\documentclass[letterpaper, 10 pt, conference]{ieeeconf}  

\IEEEoverridecommandlockouts                              

\usepackage{graphics} 
\usepackage{epsfig} 
\usepackage{subfigure}
\usepackage{mathptmx} 
\usepackage{times} 
\usepackage{amsmath} 
\usepackage{amssymb}  

\usepackage{algorithm}
\usepackage{algorithmic}
\usepackage{cite}

\newtheorem{prop}{\textbf{Proposition}}

\title{\LARGE \bf
Coordinating Large-Scale Robot Networks with Motion and Communication Uncertainties for Logistics Applications
}

\author{Zhe~Liu, Hesheng~Wang,~\IEEEmembership{Senior~Member,~IEEE,} Shunbo~Zhou, Yi~Shen and Yun-Hui~Liu,~\IEEEmembership{Fellow,~IEEE}
\thanks{This work is supported by Hong Kong ITC under Grant ITS/448/16FP, and T Stone Robotics Institute, The Chinese University of Hong Kong.}
\thanks{Z.~Liu, S.~Zhou, and Y.-H.~Liu are with the Department of Mechanical and Automation Engineering, The Chinese University of Hong Kong, Hong Kong. H.~Wang and Y.~Shen are with the Department of Automation, Shanghai Jiao Tong University, China. Corresponding author: H.~Wang.}}

\begin{document}

\maketitle
\thispagestyle{empty}
\pagestyle{empty}

\begin{abstract}
In this paper, we focus on the problem of task allocation, cooperative path planning and motion coordination of the large-scale system with thousands of robots, aiming for practical applications in robotic warehouses and automated logistics systems.
Particularly, we solve the life-long planning problem and guarantee the coordination performance of large-scale robot network in the presence of robot motion uncertainties and communication failures. A hierarchical planning and coordination structure is presented. The environment is divided into several sectors and a dynamic traffic heat-map is generated to describe the current sector-level traffic flow. In task planning level, a greedy task allocation method is implemented to assign the current task to the nearest free robot and the sector-level path is generated by comprehensively considering the traveling distance, the traffic heat-value distribution and the current robot/communication failures. In motion coordination level, local cooperative A* algorithm is implemented in each sector to generate the collision-free road-level path of each robot in the sector and the rolling planning structure is introduced to solve problems caused by motion and communication uncertainties. The effectiveness and practical applicability of the proposed approach are validated by large-scale simulations with more than one thousand robots and real laboratory experiments.
\end{abstract}
\vspace{-0.1cm}

\section{Introduction}
In logistics applications, mobile robots can be used for automatic parcel sorting in robotic warehouses, pickup and delivery in unmanned storage systems, cargo transportation in autonomous container piers, and mail service in office environments \cite{Guizzo2008,myits,Coltin2014,Digani2015}. One of the most successful applications is the kiva system \cite{Guizzo2008,Wurman2008}, where more than five hundreds of autonomous robots are implemented to transport inventory pods in warehouses. Implementing the large-scale robotic transportation system instead of human works contributes to improving the warehouse working efficiency, increasing transportation flow capacities and reducing labor costs. In such systems, task planning and motion coordination are the most challenging problems, which greatly affects the efficiency and reliability of the overall warehousing system.

\begin{figure}[!t]
\centering
\includegraphics[width=0.8\columnwidth]{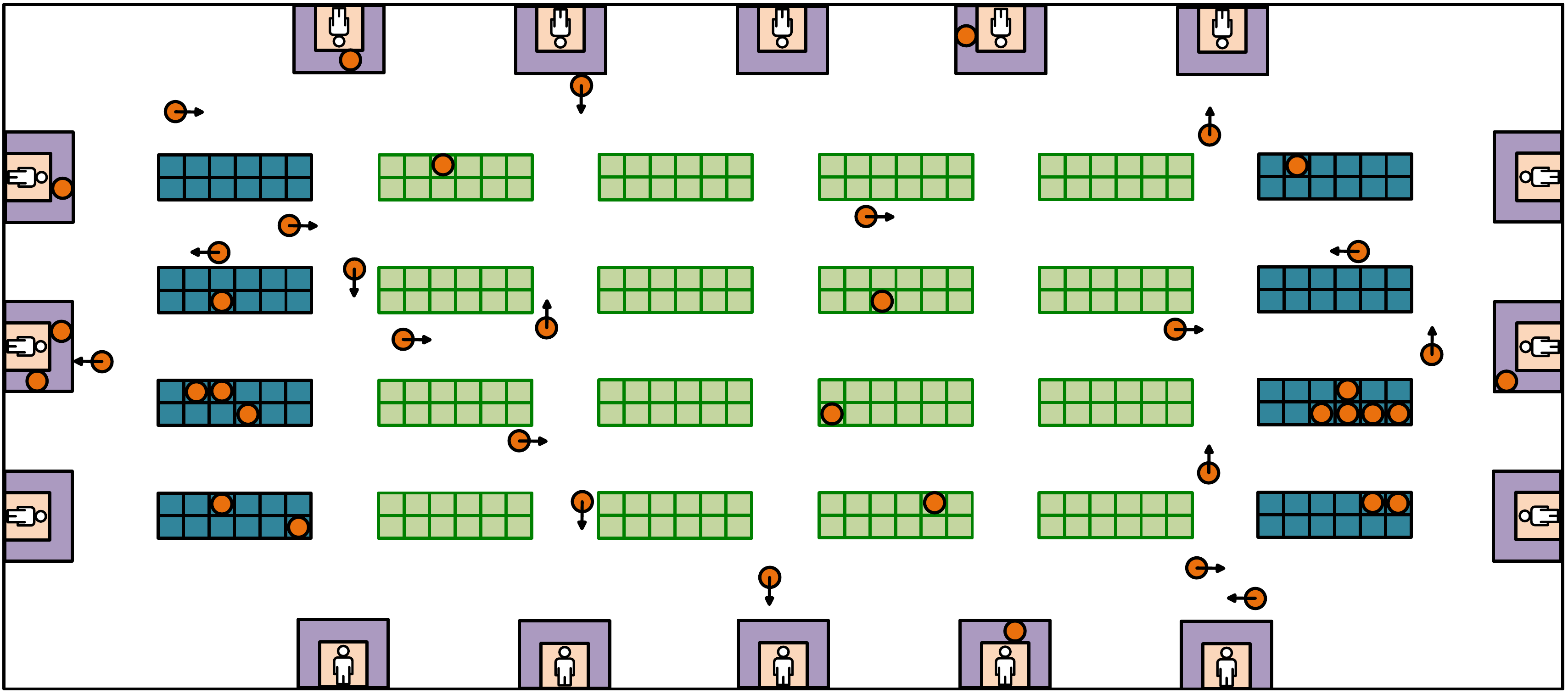}
\caption{An example of the robotic warehousing system.}
\label{figsys}
\vspace{-0.7cm}
\end{figure}

In this paper, we focus on the task planning and motion coordination problem of large-scale robotic systems in warehousing and logistics applications (as shown in Fig. \ref{figsys}). We have investigated the current situation in China and find that there is a dichotomy between the academic researches and the practical applications, and there still exist several open challenges: 1) How to build a well-constructed environment and guarantee the well-formulation of the planning instance. This problem is of great importance in the life-long planning instance, since compared with the traditional one-shot problem, the life-long problem may not always be solvable. 2) How to balance the traffic flow and achieve the path coordination purpose in the task planning stage. In large-scale systems with hundreds (or possibly thousands) of mobile robots, the density distribution of mobile robots should be balanced in order to achieve traffic flow equilibrium. So the path coordination problem is a very critical issue, which contributes to alleviating traffic congestions and avoiding shock waves. 3) How to guarantee the motion coordination performance during task accomplishing process and avoid traffic jams and robot deadlocks. These problems are very challenging since unpredictable robot motion uncertainties and temporary communication failures are un-neglected in practical applications, which will greatly affect the optimization performance of the previous planning results and even lead to task failures. 4) Increasing system reliability, reducing algorithm complexity and fulfilling real-time requirement are also challenges in large-scale robotic warehousing systems.

In this paper, we consider the above challenges in the system formulation and present a hierarchical planning and coordination approach. The main idea of the proposed approach and the contributions can be summarized as follows:

1. Considering practical warehousing environments, we present several criterions to formulate a well-formed problem of the life-long planning and coordination task, which guarantees the solvability of the life-long problem.

2. We present a hierarchical planning and coordination approach as shown in Fig. \ref{figblock}, which consists of a task planning level and a motion coordination level. In task planning level, the warehouse environment is partitioned into several sectors and a traffic heat-map is maintained to describe the current robot density distribution. In sector topology graph, the cost of traveling through a sector is defined by combining the traveling distance, the current heat-value and the current motion/communication uncertainties of the sector. Then a greedy task allocation method is implemented to assign each task to the nearest free robot and the high-level path planning is conducted in the sector topology graph. These manners aim to ensure the traffic flow equilibrium and guarantee the real-time performance. In motion coordination level, local cooperative A* algorithm in the spatial-temporal road topology is implemented in each sector and the entryway reservation mechanism is used to coordinate each pair of neighboring sectors. The online rolling planning structure is introduced in order to improve the algorithm tolerance to robot motion delays and communication failures. Finally, in the individual robot control, each robot tracks its planned road path and avoids collisions individually.

3. To the best of our knowledge, in existing literature, this paper is the first to present warehousing simulations with large-scale networks with more than one thousand robots in the presence of motion/communication uncertainties. What's more, laboratory experiments with a group of mobile robots are performed to validate the effectiveness, efficiency, robustness and practical applicability of the proposed approach.

\begin{figure}[!t]
\centering
\includegraphics[width=1\columnwidth]{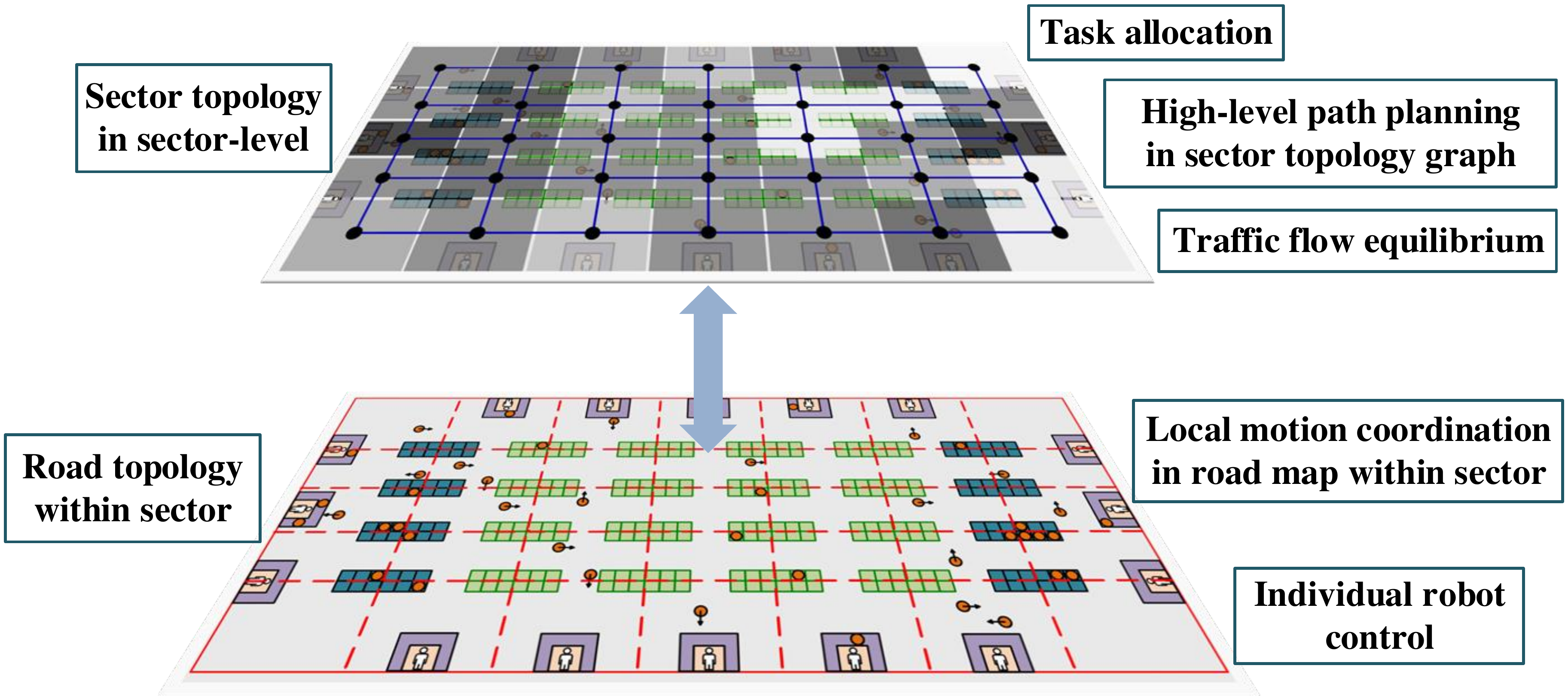}
\caption{System structure of the proposed hierarchical approach.}
\label{figblock}
\vspace{-0.5cm}
\end{figure}

\section{Related Work}
A directly relevant research topic of the problem studied in this paper is the multi-robot path planning and motion coordination problem, which has been investigated actively for many years. Traditional researches can be mainly divided into the centralized approach and decentralized approach. Generally speaking, the centralized approaches \cite{Yu2016} aim to find out the optimal solution (or increase the solution optimality as much as possible) and decrease the computation complexity, while the decentralized approaches \cite{Smolic-Rocak2010} mainly focus on increasing system flexibility, achieving distributed implementations, fulfilling real-time requirements and resolving local traffic jams and robot deadlocks. Centralized approaches can be further divided into the reduction-based methods \cite{Yu2016} and the A*-based methods \cite{Wurman2008}. In reduction-based methods, the path planning problem is firstly transformed into some well-studied mathematical programming problems \cite{Yu2016,Ma2016}, such as the time-expanded network flow problem, then the optimal solution (or near-optimal solutions) can be obtained by existing programming algorithms (such as the integer linear programming). However, typically, these methods are only efficient to optimize the makespan index on small-scale problems, and the algorithm complexity in problem translation processes also can not be neglected. Among A*-based methods, the hierarchical cooperative A* algorithm and its extensions \cite{Roozbehani2011} have good realtime performance and are promising in solving large-scale problems, although they can not guarantee the optimal performance. Recently, searching tree based algorithms \cite{Sharona2015,Ma2017-2} are presented to decrease the cooperative A* searching complexity and improve the optimality, however, the performance in large-scale problems has not been validated. Decentralized approaches can be further divided into the local coordination based method and the prioritized planning based method. Local coordination based methods are usually based on pre-defined traffic rules. When confliction occurs, robots will replan their local routes \cite{Regele2006}, coordinate their velocities \cite{Berg2011}, or change their time schedules \cite{Smolic-Rocak2010} to avoid deadlocks and traffic jams. In prioritized planning based methods \cite{Cap2015}, each robot has a priority value and plans its individual path iteratively. The path coordination is achieved by continuously re-planning operations of the robots with lower priorities. Decentralized approaches are mostly incomplete and un-reliable in practical applications, and the high requirements for the robot perception and communication abilities also limit these approaches' practical applicabilities.

Please note that the approaches mentioned above mostly concentrate on the one-shot problem, where all the tasks are predefined and each robot needs to deliver its allocated task while avoiding conflictions with other robots. All the task positions are different from each other and the system terminates once all robots arrive at their destinations. However, the life-long planning and coordination problem may not always be solvable and, compared with finding the optimal solution, ensuring the real-time performance and improving the system reliability are more important. In \cite{Wurman2008}, the authors incorporate the utility- and auction-based method to achieve online task allocations, and introduce the space reservation mechanism and mono-direction highways to avoid robot collisions \cite{Roozbehani2011}. However, in the presence of robot motion delays, this approach needs to re-plan the path of corresponding robots frequently, and the traffic jams and deadlocks can not be avoided completely. Recently, the authors in \cite{Ma2017} solve the life-long planning and coordination tasks based on the token passing mechanism and the cooperative A* method. However, since the planned paths and assigned tasks of all robots should be shared synchronously and sent to the robot who has the token currently, the computational complexity, storage complexity and scalability of this approach can not be guaranteed. Furthermore, this approach is based on the assumptions of ideal vehicle motion control performance and communication performance, ignoring the communication and motion uncertainties also limits its practical applicability.

\section{Well-Formulation of Life-Long Planning}
Consider a warehouse environment as shown in Fig. \ref{figsys}, where the blue blocks represent the robot station area (such as the charging station or parking station in a warehouse), the green blocks represent the task pickup station area (such as the inventory pod area) and the other blocks represent the working station area (such as the cargo delivery area). All the free robots stay in the robot station area and wait for the task assignments. Once a new published task is assigned to a free robot, the robot will move from the robot station to the corresponding task pickup station to pick up the task and then deliver it to the corresponding working station. In addition, the robot should avoid collisions with other robots during the movement and return to one of the free robot stations after accomplishing the assigned tasks.

A well-constructed warehouse environment should fulfill the following criterions:
\begin{itemize}
\item[C1] There are no fewer robot stations than robot number.
\item[C2] Between any two stations (including the robot stations, task pickup stations and working stations), there exists at least one path that traverses no other stations.
\end{itemize}

A new assigned task should fulfill the following criterions:
\begin{itemize}
\item[C3] For a new task, its corresponding task pickup station and working stations are different with those of all the unaccomplished previous assigned tasks.
\item[C4] If C3 can not be fulfilled, the new task should be put into a waiting list to wait for the accomplishment of the conflicting task.
\end{itemize}

The delivering robot should fulfill the following criterion:
\begin{itemize}
\item[C5] Once the assigned task is accomplished, the robot should either return to one of the free robot stations, or move away from the working station to the task pickup station area to deliver a new assigned task.
\end{itemize}

\begin{prop}
The Life-long task planning and motion coordination problem is always solvable, i.e., the problem is well-formulated, if the above criterions C1-C5 are fulfilled.
\end{prop}
\begin{proof}
C1-C2 ensure that robots have enough parking stations and a robot which stays in a station will not conflict with others. C3-C5 ensure that a robot which is delivering a previous assigned task will not conflict with others which are going to accomplish new tasks. If a robot can not find out a collision-free path, it can stay in a station and wait for accomplishments of all the conflicting tasks, then it can find out at least one free path to deliver its task. So in the worst case, robots can deliver their tasks one by one according to the task assignment order, i.e., each robot has a priority value defined by its task assignment order, the robot with the lower priority will stay in the robot station until all the robots with higher priorities have accomplished their tasks and returned to robot stations. Thus the life-long task planning and motion coordination problem is always solvable.
\end{proof}
\section{Hierarchical Planning and Coordination}

\subsection{Centralized High-Level Task Planning}
The objective in the task planning level is to assign each new published task to a free robot and plan the corresponding sector-level path. 

\subsubsection{Environment Partitioning}
$\ $

As shown in the bottom of Fig. \ref{figblock}, we divide the environment into several sectors by utilizing the following criterions: 1. There is only one intersection in each sector; 2. For two neighboring sectors, there is only one entryway from one sector to the other, and vice versa; 3. For each sector, there is at least one path from each entryway to each exit-way, and the path does not traverse any station in the sector. These criterions aim to guarantee the solvability of the road-level path planning problem within each sector and minimize the complexity of the low-level motion coordination problem.

\subsubsection{Heat-map Generation}
$\ $

A traffic heat-map is generated and maintained to describe the current robot density distribution. The heat-value of each sector is defined as the ratio of the current number of the robots in the sector to the maximum allowable robot number of the sector (which is defined to be linearly proportional to the size of the sector). Please note that the robots which stay in robot stations without any assigned task will not be considered in the heat-value calculation.


\subsubsection{Sector-Level Topology Graph Generation}
$\ $

Based on the partitioning results, the environment can be described by a sector topology graph $G$ (as shown in the upper part of Fig. \ref{figblock}), where each vertex $v_i$ represents a sector $s_i$ and each edge $(v_i,v_j)$ represents that there is an entryway from sector $s_i$ to $s_j$. The edge weight $w_{ij}$ is defined as
\begin{equation}
w_{ij}=d_{ij}\times (1+k_hh_j+k_ll_j),\nonumber
\end{equation}
where $d_{ij}$ is the travelling distance from the center of sector $s_i$ to the center of sector $s_j$, $h_j$ is the current traffic heat-value of sector $s_j$, $l_j$ represents the normalized current number of abnormal robots in sector $s_j$, $k_h$ and $k_l$ are positive weighting coefficients. Abnormal robots include the robots with motion delays and the robots with communication failures.

\subsubsection{Task Allocation and Sector-Level Path Planning}
$\ $

In the life-long task planning problem, all the transportation tasks are published online. An unassigned task list is maintained and updated in the task planning system, which contains all the unassigned tasks with the order of the task priorities. The task priority can be defined by considering the publishing time, importance or urgency of each task. Then the greedy task allocation method is implemented to assign each task to the nearest free robot and the corresponding high-level path planning is conducted in the sector topology graph $G$ based on the A* algorithm. Fig. \ref{figtask} shows an example of the sector-level path planning.
\vspace{-0.2cm}
\begin{figure}[!h]
\centering
\includegraphics[width=0.9\columnwidth]{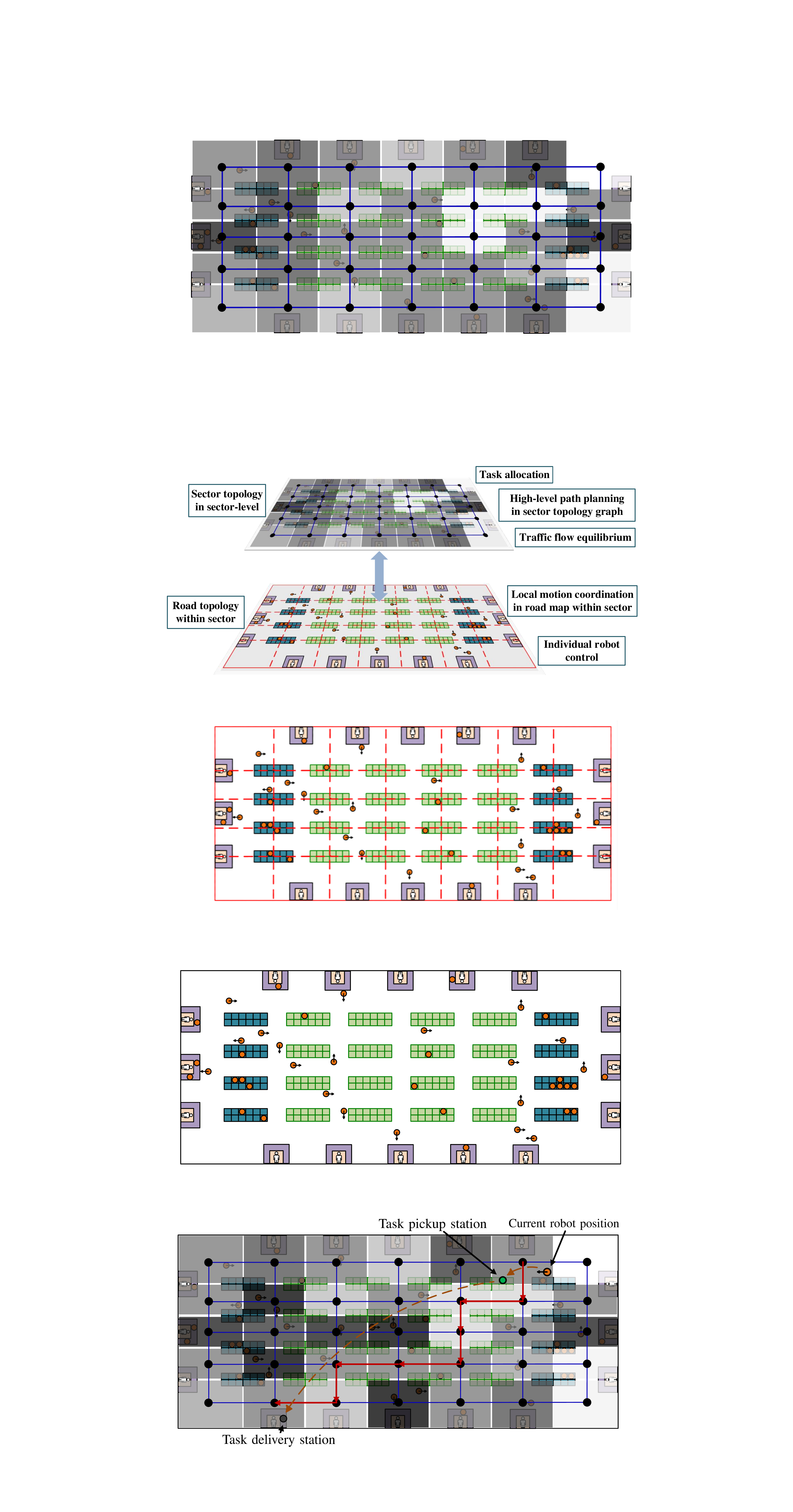}
\caption{An example of the sector-level path planning. The black nodes and blue lines represent the sector-level topology graph, where the gray scale of the sector is darker the traffic heat-value is larger, and the red lines represent the sector-level path planning result.}
\label{figtask}
\vspace{-0.3cm}
\end{figure}

\subsubsection{Discussions}
$\ $

The strategies mentioned above aim to ensure the sector-level traffic flow equilibrium and guarantee the real time requirements in the centralized task planning level.

Balancing the traffic flow distribution in the whole environment contributes to reducing the probability of robot collisions and congestions, increasing transportation flow capacities and improving the warehouse working efficiency. Only considering the shortest path in each task planning process can not ensure the expected optimization performance, since the robot congestions (especially in crossroads or trunk roads) can not be neglected and will greatly affect the real performance of the previous planning results. Cooperative A* algorithm in the spatial-temporal road topology \cite{Roozbehani2011,Sharona2015} is commonly used in existing approaches to plan a collision-free path for each robot, however, this approach requires that each robot should follow its time schedule strictly. In practical applications, this requirement usually can not be ensured since the manipulation time in task loading and unloading processes, the movement time from one place to another place and the waiting time in working stations usually can not be accurately predicted or controlled.

In this paper, we introduce the traffic heat-value into the travelling cost to achieve the traffic flow equilibrium purpose. Although this approach may not avoid traffic jams completely, we can greatly put up the travelling cost of the sector with traffic jams or even close the sector temporarily in the worst case. Then the traffic congestion can be alleviated gradually and the deadlocks can be avoided.

\subsection{Decentralized Low-Level Motion Coordination}
Motion coordination is achieved in each sector in a decentralized manner, this will greatly reduce the complexity of the coordination problem while maintaining the solvability of the whole life-long problem. The objective in this level is to plan the road-level local path of each robot travelled through the section and coordinate the motion of each robot to avoid collisions. 
\vspace{-0.2cm}
\begin{figure}[!h]
\centering
\includegraphics[width=0.8\columnwidth]{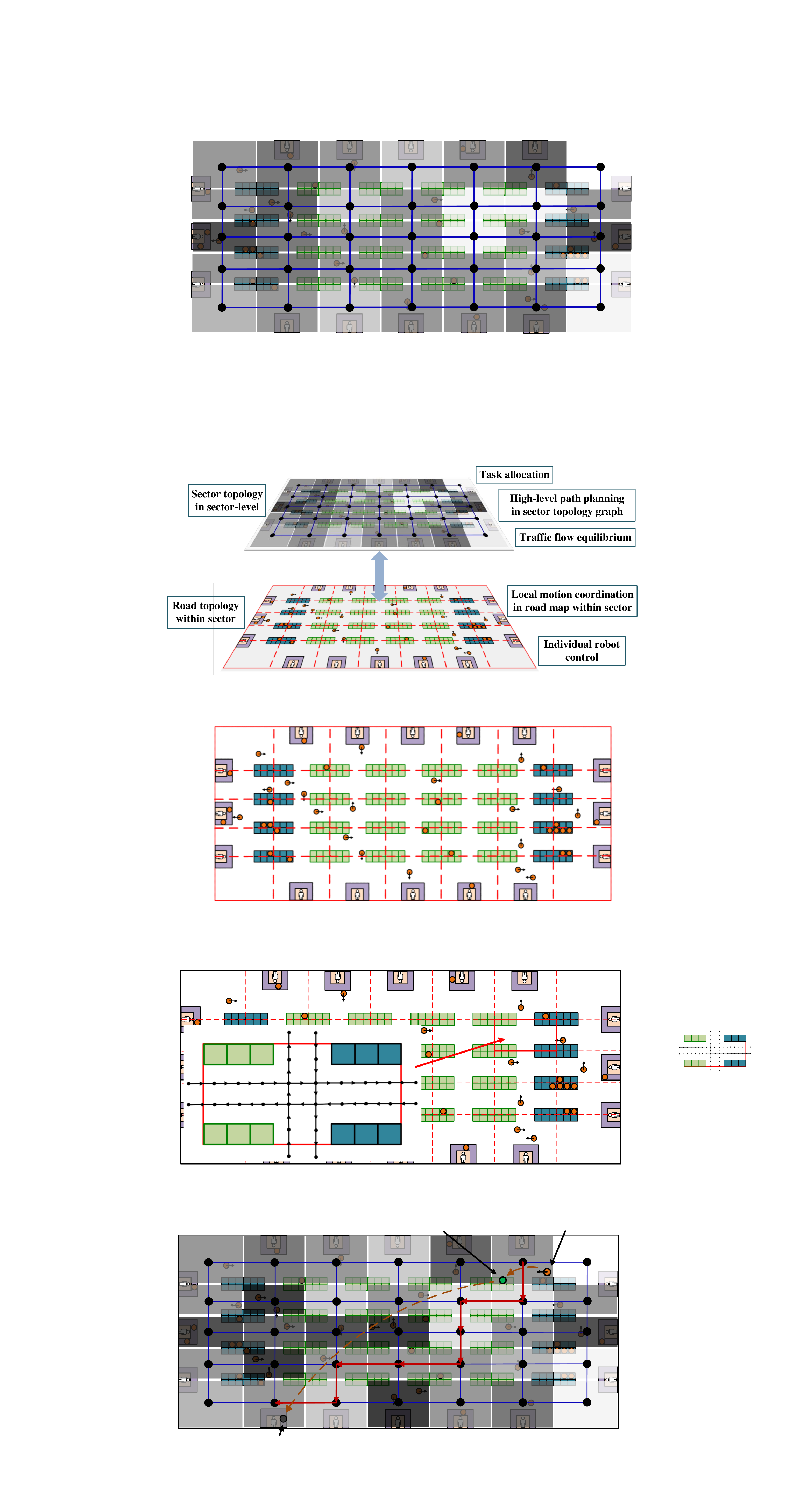}
\caption{An example of the road topology map in a sector.}
\label{figroad}
\vspace{-0.3cm}
\end{figure}
\subsubsection{Road Segmentation and Road Topology Map Generation}
$\ $

The roads in each sector are divided into several segments and a road topology map is generated and represented as a graph $M$, where each vertex $u_i$ represents a place $q_i$, each edge $(u_i,u_j)$ represents that there is a road segment from $q_i$ to $q_j$ that does not travel through any other vertex, and the edge weight is defined by the travelling distance of the corresponding road segment. Inspired by the highway concept presented in \cite{Roozbehani2011}, we define the road segments in task and robot station areas as mono-direction roads in order to reduce congestions and improve the traffic speed. Fig. \ref{figroad} shows an example of the road topology map within a sector.

\subsubsection{Road-Level Path Planning and Local Motion Coordination}
$\ $

Based on the road topology map, the cooperative A* algorithm and the conflict-based searching strategy in the spatial-temporal space \cite{Sharona2015} are implemented in each sector to plan a collision-free road-level path (within the sector) for each robot which travels through the sector. This approach ensures the optimal performance of the local path planning.

Since the robots in each sector changes dynamically (i.e., at each time step, a new robot may move into the sector or a previous robot may move out), we introduce the online roll planning structure to replan the local path of each robot at each time step, i.e., in each time step, the planner of each sector plans the whole road-level path for each robot in the sector, but each robot only moves one step (travels through the first road segment in its planned road-level path or stays in its current position to avoid conflicts with another robot), and this process will be repeated iteratively. The introduction of the rolling planning structure also contributes to solve the motion delay problem caused by uncertainties mentioned in Section III.A.5, since the path of the delayed robots will be replanned in the next time step.

An entryway reservation mechanism is introduced to coordinate each pair of neighboring sectors. If a robot has reached the exit-way of the current sector and is going to enter the entryway of the next sector, the robot should firstly apply for the entryway occupancy authority of the next sector, then the entryway will be reserved for the robot and can not be applied by other robots, until the entryway occupancy authority has been released.

\subsubsection{Strategies for resolving communication failures}
$\ $

The communication connection between robots and the sector planner/controller may be temporarily disabled due to bandwidth limitations, route handoff and re-establishment operators, or hardware malfunctions. By implementing the heart-single based method \cite{mysmc}, the state of each robot in the sector can be monitored by the sector controller and communication failures can be detected in real time.

We design a K-step redundancy mechanism to solve the communication failure problem: If a communication failure occurs, the failed robot will continue to follow its road-level path planned in the previous time step, while waiting the re-establishment of the communication connection. In order to ensure the safety, if the communication cannot be reconstructed immediately, the failed robot will move up to K steps and then stops to wait for the communication reconstruction. What's more, if the failed robot reaches the exit-way of its current sector, it will also stop to ensure safety. In the meantime, the sector planner/controller will estimate the maximum motion range (the K-step range) of the failed robot and close the corresponding area to prevent other robots from entering, and command the robots which are currently located in that area to stop. Since in the individual robot navigation level, each robot will detect collisions by its onboard sensors (such as range sensors or impact sensors), safety can be ensured. Finally, once the communication connection is reconstructed, the sector planner/controller will replan the road-level path for all robots.

The proposed K-step redundancy mechanism aims to reduce the impact of the communication failure on the traffic efficiency while ensuring safety. Since most of the communication failures in practical applications are temporary failures and can be recovered in a short time (during the K step movements of the failed robot), so the impact can be eliminated and the system performance can be maintained. 

\section{Simulations and Experiments}
\begin{figure}[!h]
\centering
\includegraphics[width=0.9\columnwidth]{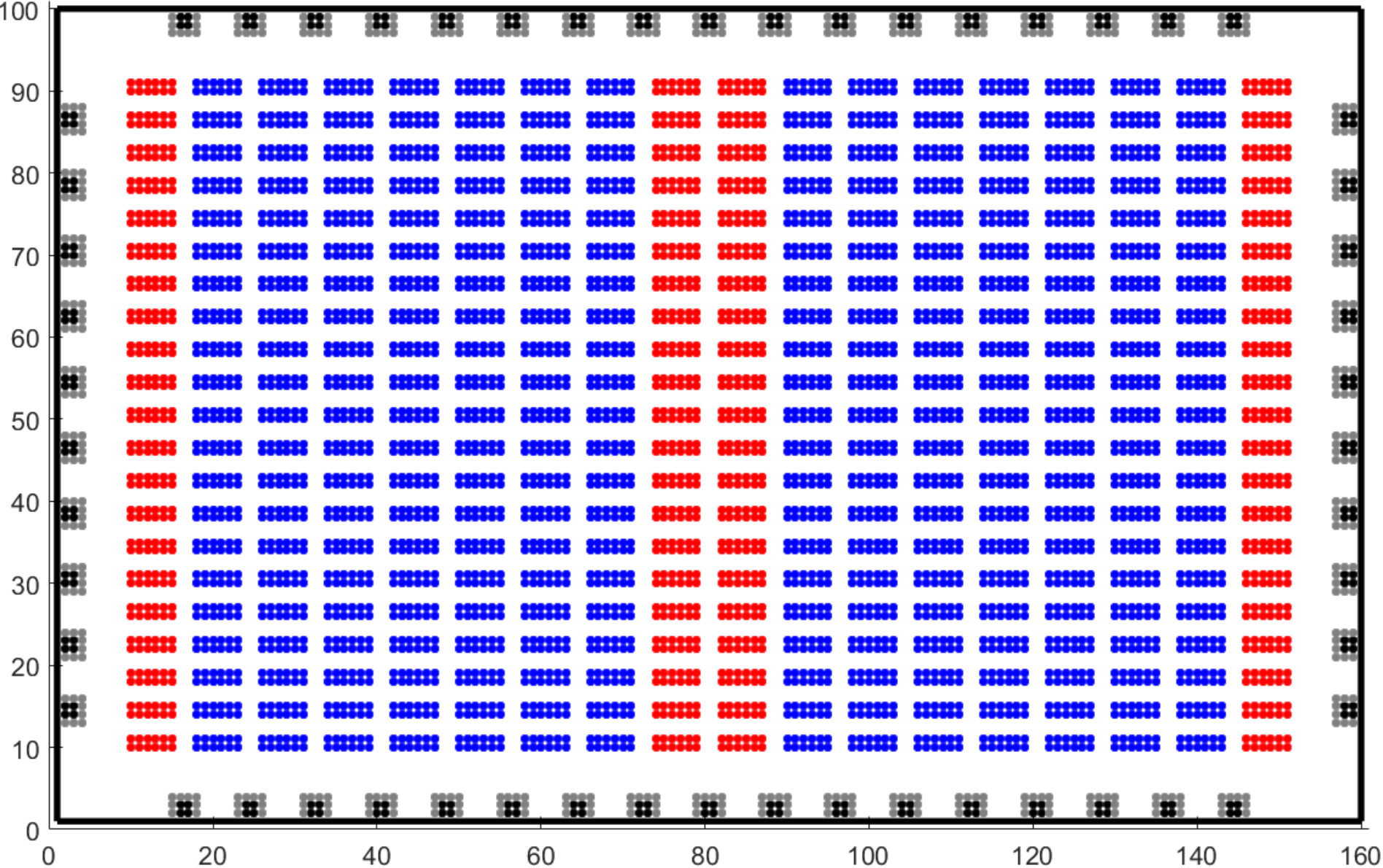}
\caption{Simulation environment. The blue blocks represent the task station areas (in each block, there are 12 nodes which represent the 12 task pickup stations in each task area), the red blocks represent the robot station areas (in each block, there are 12 nodes which represent the 12 robot stations in each robot station area) and the grey blocks represent the working station areas (in each block, there are 8 nodes which represent the 8 working stations in each working area). The size of one station is set to $1m\times 1m$, which represents the minimum space occupied by one robot which ensures safety.}
\label{figsimap}
\vspace{-0.2cm}
\end{figure}
Matlab simulations are conducted with Intel 3.4GHz Core i7-6700 CPU and 16G RAM. As shown in Fig. \ref{figsimap}, the map size is $160m\times100m$, there are 1008 robot stations, 3528 task stations and 432 working stations. According to Proposition 1, this environment can be used to conduct large-scale simulations with more than one thousand robots. In simulations, we consider 1008 robots and 3000 online tasks, and set $K=3$, $k_h=10$, $k_l=50$. Since there is no existing open algorithm which achieves large-scale simulations with thousands of robots in the presence of motion/communication uncertainties, we only provide simulations of the proposed approach in different task publishing frequencies and uncertainty levels. In each time step: 1) In order to simulate robot motion uncertainties, $f_m$ percent of robots are controlled to stay in their current positions and ignore the command from the controller/planner. 2) In order to simulate communication failures, one of the normal robots fails and the previous failed robots are recovered with a probability of $f_c$.
  
We first set $f_m=0.01$, $f_c=0.3$ and test the proposed approach with different task publishing frequencies. Simulations in each task frequency are repeated for 5 trials, the mean and standard deviation of each index are recorded in Fig. \ref{figs1}. The saturation phenomenon can be found in each index, which results from the maximum task capacity of the simulation environment (according to Proposition 1, the task capacity has an upper bound once the environment is defined). Then we set the task frequency to 5 (slightly above the estimated maximum task capacity) and test the proposed approach with different uncertainty levels. We set $f_m=0.01$, $f_c=0.3$ in level 1, $f_m=0.05$, $f_c=0.25$ in level 2, $f_m=0.02$, $f_c=0.2$ in level 3, $f_m=0.025$, $f_c=0.15$ in level 4 and $f_m=0.03$, $f_c=0.1$ in level 5. Results in Fig. \ref{figs1} and \ref{figs2} validate the effectiveness of the proposed approach in large-scale problems with large robot motion/communication uncertainties. The saturated computation time is nearly 0.7s and remains unchanged under different uncertainty levels. The task waiting time and finish time increase linearly and slightly with increasing uncertainty level. What's more, the heat-value of each sector changes slightly and remains in the acceptable value interval, these demonstrate that robot congestion is avoided and traffic flow equilibrium purpose is achieved, even in the presence of large robot motion/communication uncertainties.
\vspace{-0.15cm}
\begin{figure}[!h]
\centering
\includegraphics[width=0.95\columnwidth]{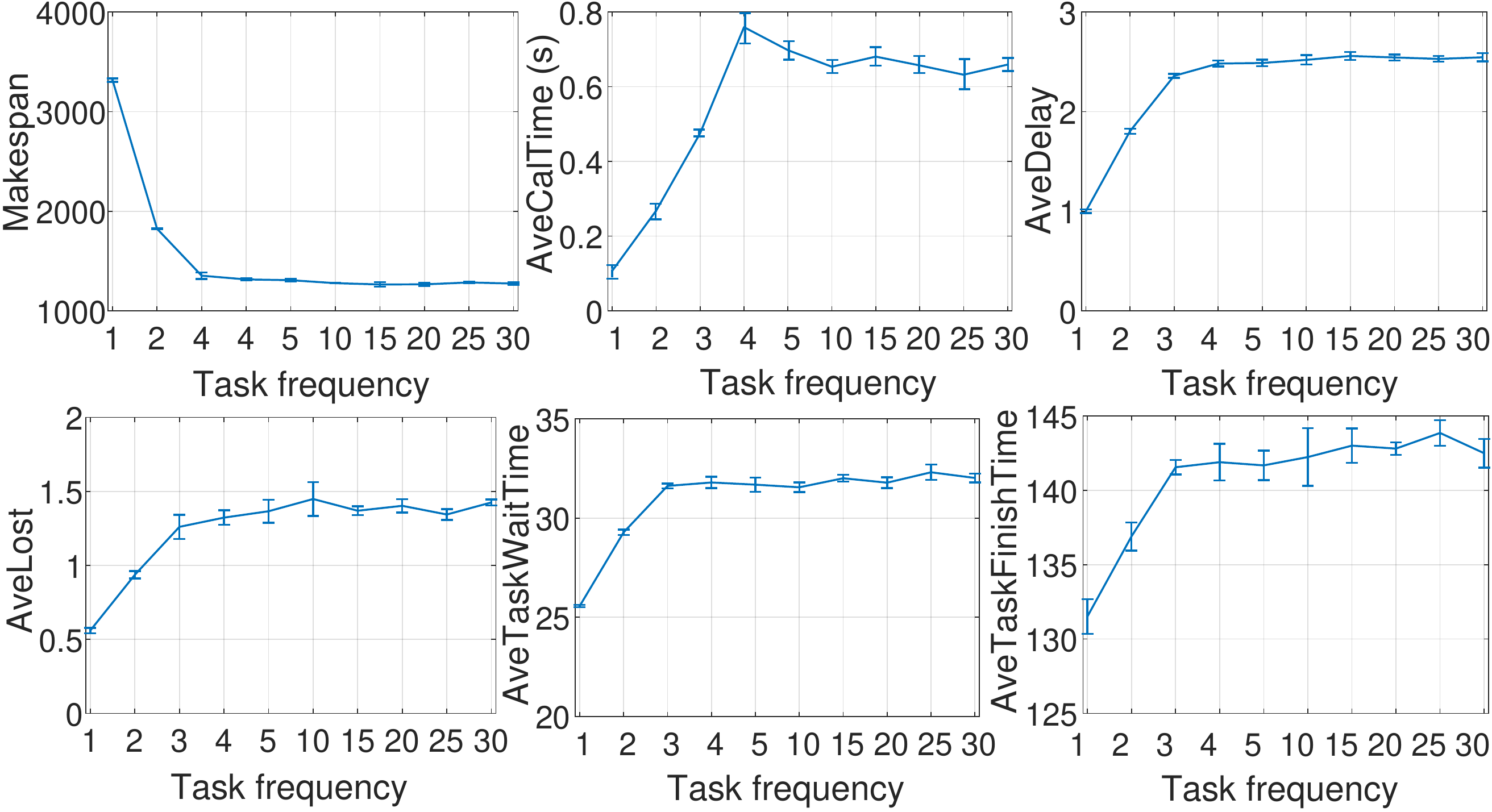}
\caption{Simulation results with 1008 robots and 3000 tasks: under different task frequencies. \emph{Task frequency} denotes the number of tasks published in each time step, \emph{Makespan} represents the time step required to finish all the 3000 tasks, \emph{AveCalTime} is the average computation time in each step, \emph{AveDelay} and \emph{AveLost} are the average number of robots with motion delays and communication failures in each step, \emph{AveTaskWaitingTime} and \emph{AveTaskFinishTime} represent the average step required from the publishing instant to the pickup instant and the completion instant of each task.}
\label{figs1}
\vspace{-0.22cm}
\end{figure}

\vspace{-0.25cm}
\begin{figure}[!h]
\centering
\subfigure{\includegraphics[width=0.95\columnwidth]{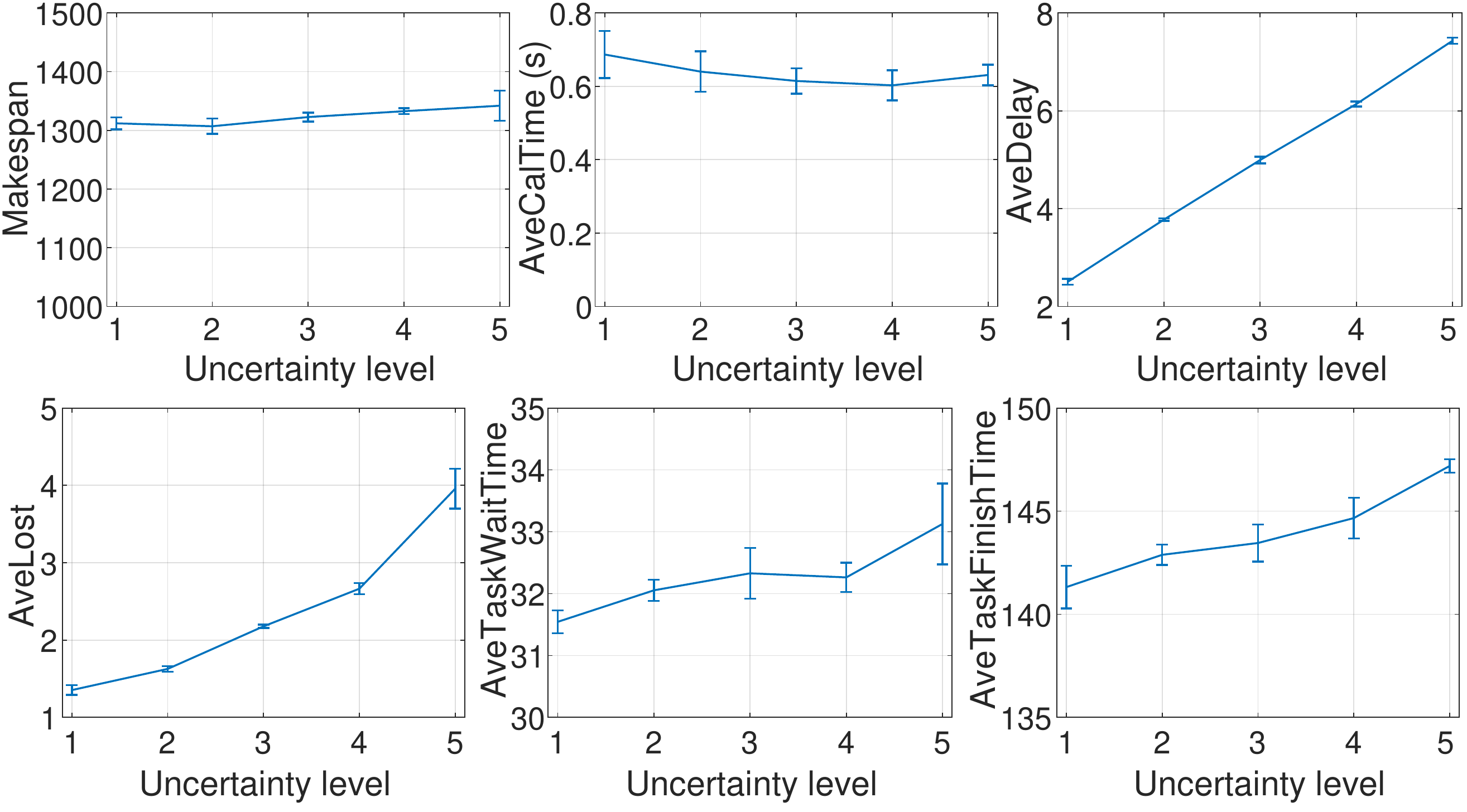}}
\subfigure{\includegraphics[width=0.9\columnwidth]{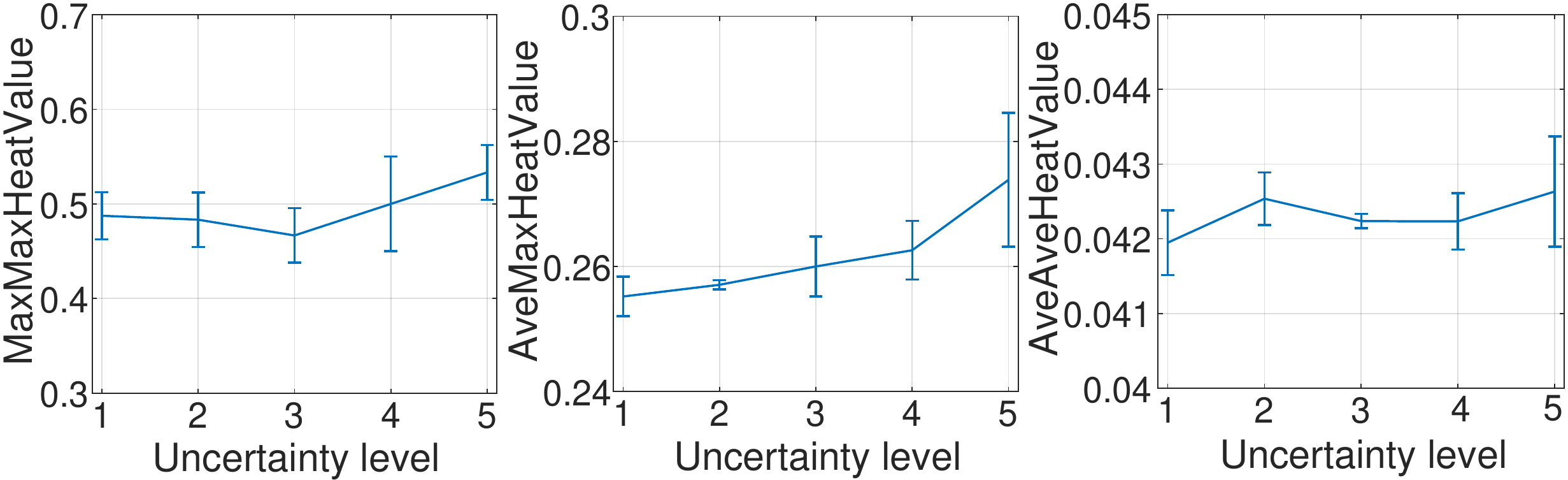}}
\caption{Simulation results with 1008 robots and 3000 tasks: under different motion and communication uncertainty levels. \emph{MaxMaxHeatValue} denotes the maximum heat-value of all sectors in all time steps, \emph{AveMaxHeatValue} represents the average of the maximum heat-value of all sectors in each time step, \emph{AveAveHeatValue} represents the average of the average heat-value of all sectors in each time steps. Other indices are the same as in Fig. \ref{figs1}.}
\label{figs2}
\vspace{-0.3cm}
\end{figure}


\begin{figure}[!h]
\centering
\includegraphics[width=0.7\columnwidth]{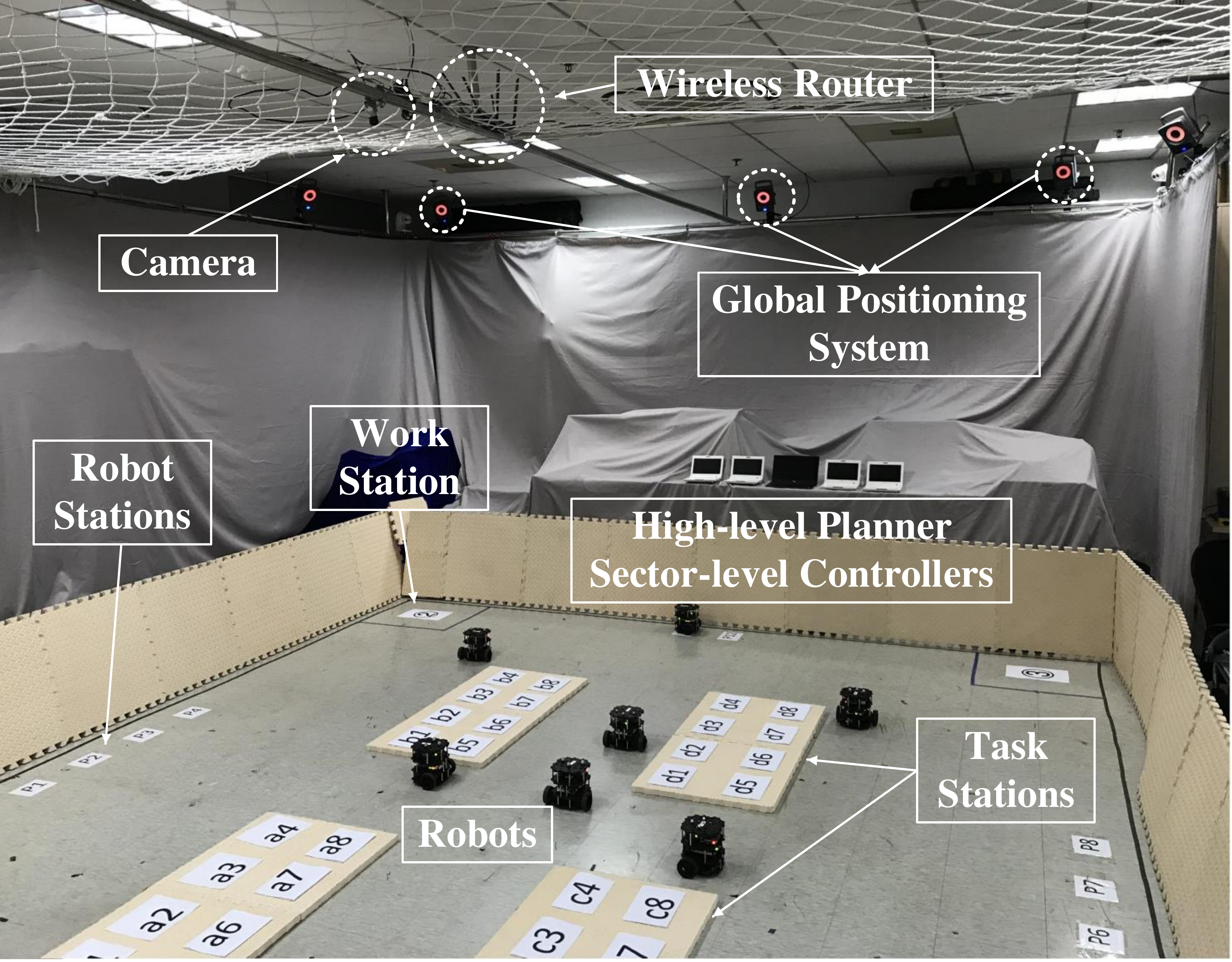}
\caption{Experiment system.}
\label{figexpsys}
\vspace{-0.2cm}
\end{figure}

We further conduct laboratory experiments on a group of seven mobile robots (TurtleBot-3 Burger). As shown in Fig. \ref{figexpsys}, a global positioning system is constructed to localize each robot and a wireless communication network is established to achieve communications between robots, the task planner and sector controllers. One laptop with Intel Core i7-6500U CPU 2.59GHz 8G RAM and four laptops with Intel Atom N270 CPU 1G RAM 1.6GHz are used to realize the centralized planning and decentralized coordination algorithms. Experiment environment is shown in Fig. \ref{figexpmap}. Note that temporary communication blocks and motion un-accuracies are inherently existed in the experiment system. In experiments, 12 tasks are generated online with random pickup stations and working stations, as shown in Tab. \ref{tabexp}, where $T_O^i$, $T_T^i$ and $T_W^i$ represent the publishing time, pickup station and working station of each task respectively. Initial robot positions are also generated randomly, where $R_1$ in $p_7$, $R_2$ in $p_{11}$, $R_3$ in $p_4$, $R_4$ in $p_{12}$, $R_5$ in $p_6$, $R_6$ in $p_5$, $R_7$ in $p_9$.

\begin{figure}[!h]
\centering
\includegraphics[width=0.65\columnwidth]{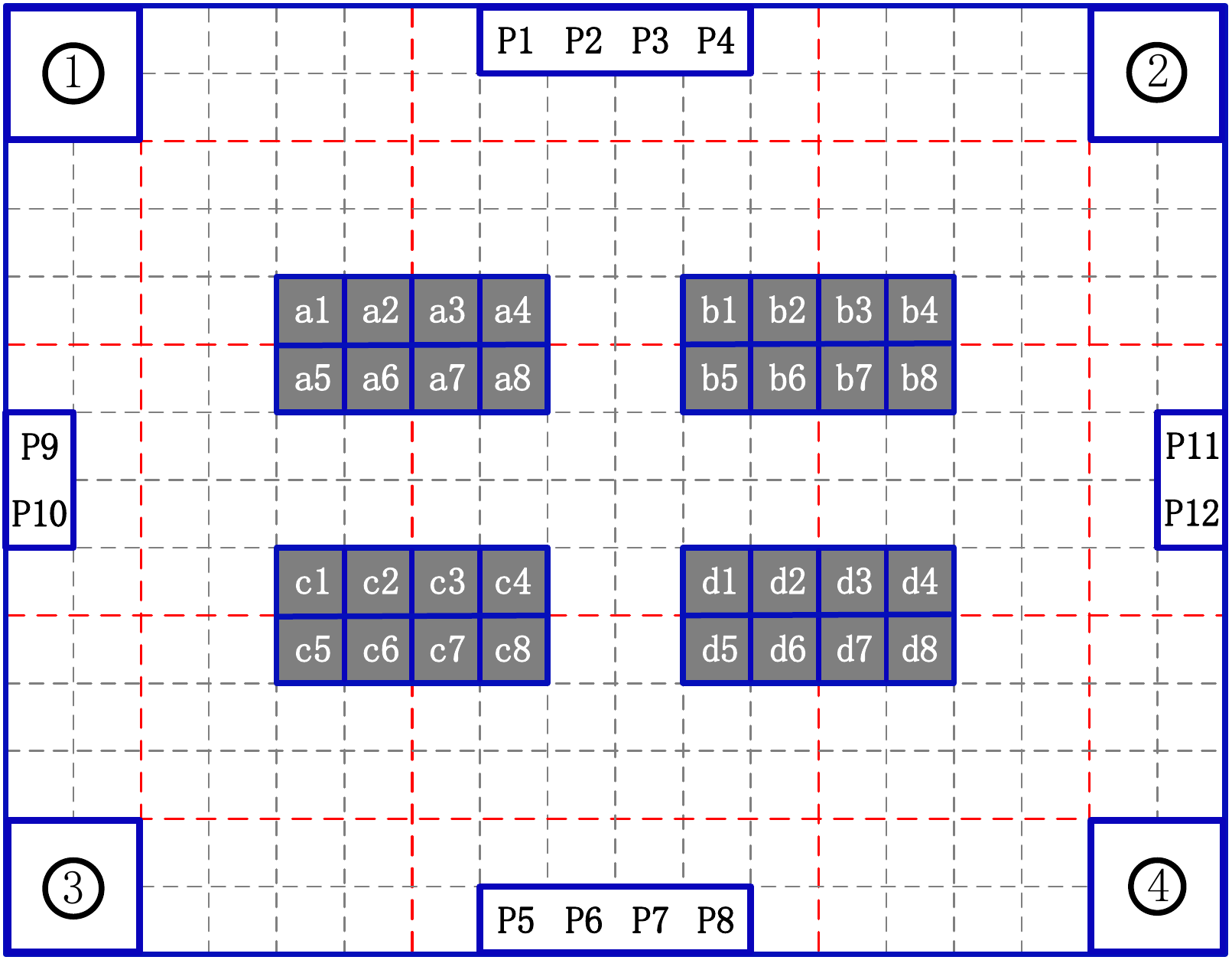}
\caption{Experiment environment. The size is $5.4m\times 4.2m$, with 32 task pickup stations (there are 4 task station areas, i.e., a-d, and in each area, there are 8 task stations), 12 robot stations (
$p_{1}$-$p_{12}$) and 4 working areas (in each area, there are 4 working stations). The size of one station is set to $0.3m\times0.3m$ with considering the robot size $0.14m\times 0.18m$. The environment is partitioned into several sectors as divided by the red lines.}
\label{figexpmap}
\vspace{-0.2cm}
\end{figure}

\begin{table}[!h]
\renewcommand{\arraystretch}{1.0}
\caption{Experiment Results}
\label{tabexp}
\vspace{-0.2cm}
\centering
\begin{tabular}{cccccccc}
\hline
\hline
 & $T_O^i$ & $T_T^i$ & $T_W^i$ & $T_R^i$ & $T_A^i$ & $T_B^i$ & $T_F^i$\\
\hline
$T_1$ & $2s$ & $d_{8}$ & $4$ & $R_5$ & $2s$ & $12s$ & $24s$\\
$T_2$ & $2s$ & $a_7$ & $2$ & $R_7$ & $2s$ & $8s$ & $37s$\\
$T_3$ & $2s$ & $c_2$ & $3$ & $R_1$ & $2s$ & $18s$ & $35s$\\
$T_4$ & $2s$ & $a_2$ & $1$ & $R_3$ & $2s$ & $17s$ & $27s$\\
$T_5$ & $2s$ & $d_3$ & $1$ & $R_4$ & $2s$ & $7s$ & $40s$\\
$T_6$ & $8s$ & $b_{4}$ & $4$ & $R_2$ & $8s$ & $21s$ & $59s$\\
$T_7$ & $8s$ & $d_{6}$ & $4$ & $R_6$ & $8s$ & $21s$ & $35s$\\
$T_8$ & $14s$ & $b_{8}$ & $2$ & $R_5$ & $24s$ & $42s$ & $57s$\\
$T_9$ & $14s$ & $a_{5}$ & $1$ & $R_3$ & $27s$ & $43s$ & $64s$\\
$T_{10}$ & $20s$ & $a_{5}$ & $3$ & $R_1$ & $35s$ & $52s$ & $81s$\\
$T_{11}$ & $20s$ & $b_{2}$ & $2$ & $R_6$ & $35s$ & $62s$ & $83s$\\
$T_{12}$ & $26s$ & $a_{2}$ & $1$ & $R_7$ & $37s$ & $64s$ & $80s$\\
\hline
\hline
\end{tabular}
\vspace{-0.2cm}
\end{table}

Experiment results are also shown in Table \ref{tabexp}, where $T_R^i$ denotes the robot which is assigned to deliver the task, $T_A^i$, $T_B^i$ and $T_F^i$ represents the real time instants in the experiment when the task is assigned to a robot, picked up and accomplished by the robot, respectively. Results show that each task is immediately assigned to the nearest free robot once published or to the earliest free robot once it has finished its previous task, thus increasing the working efficiency and reducing the transportation cost. Fig. \ref{figexpresult} shows the real robot trajectories recorded by the global positioning system. We can find that robots can track their road-level paths with an acceptable navigation accuracy. Fig. \ref{figexpscan} shows two consecutive scenarios as an example to demonstrate the motion coordination performance of the proposed approach. More details can be found in the proposed video attachment.

\begin{figure}[!h]
\centering
\includegraphics[width=0.65\columnwidth]{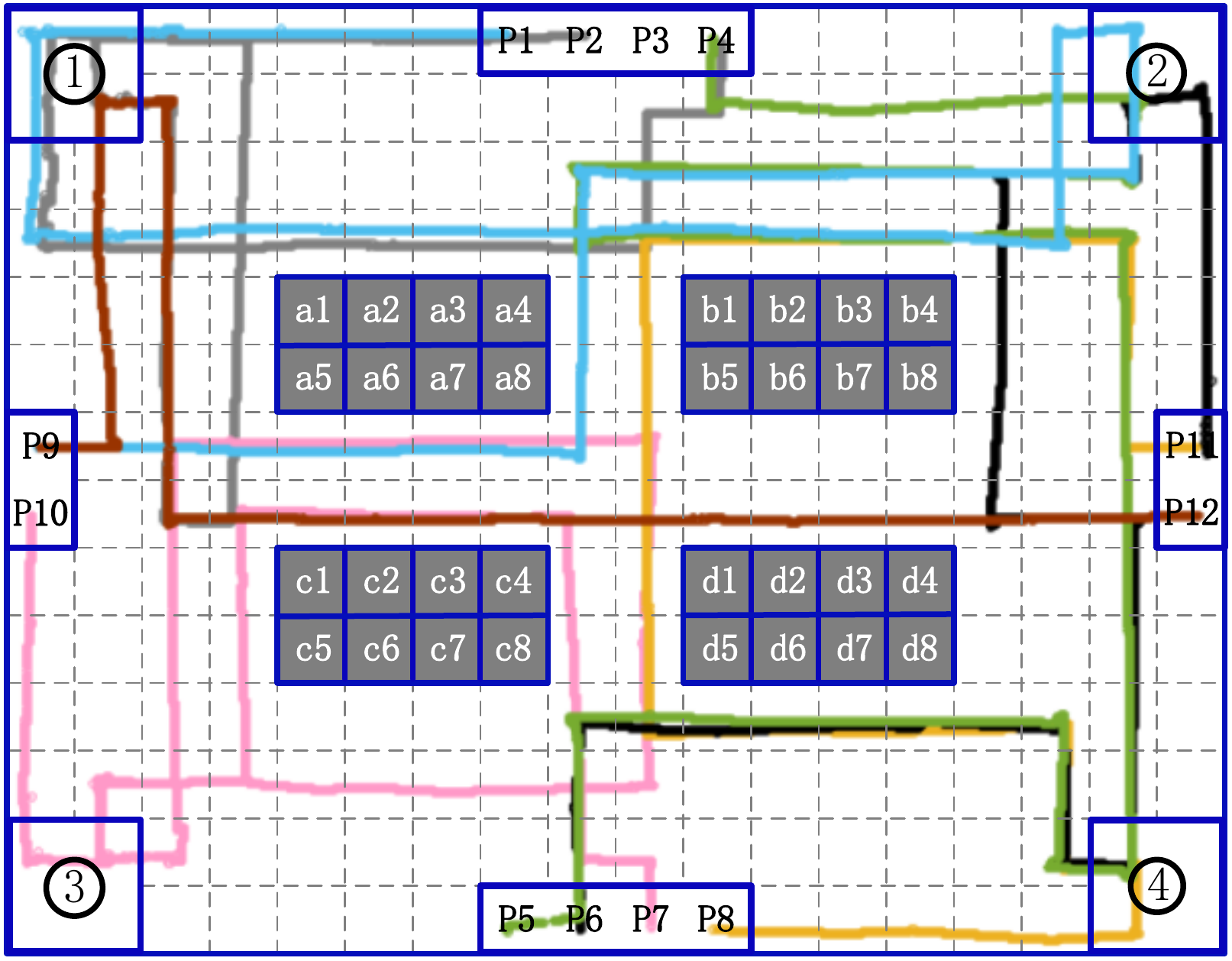}
\caption{Real robot trajectories recorded by the global positioning system.}
\label{figexpresult}
\vspace{-0.2cm}
\end{figure}

\begin{figure}[!h]
\centering
\includegraphics[width=0.75\columnwidth]{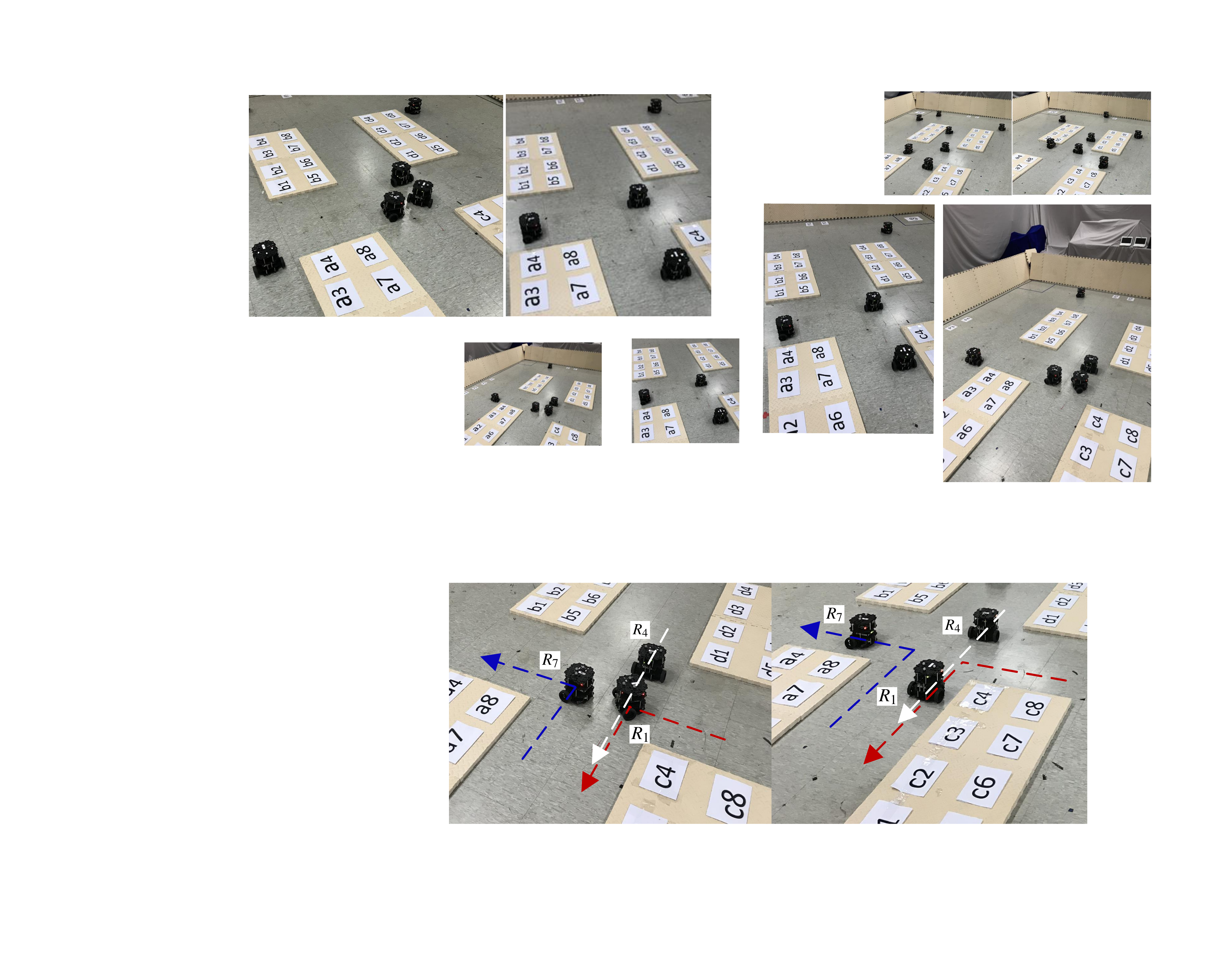}
\caption{Experiment scenarios in $t=14s$ (left) and $t=17s$ (right), where the motions of three robots are coordinated in an intersection to avoid collisions.}
\label{figexpscan}
\vspace{-0.2cm}
\end{figure}

\section{Conclusion}
In this paper, we propose a hierarchical approach to solve the life-long task planning and motion coordination problem with large-scale robot networks. Aiming for practical applications, motion uncertainties and communication failures are taken into account in system design, and the traffic flow equilibrium purpose is considered in order to ensure the working efficiency. Simulations with more than one thousand robots validate the effectiveness, scalability and robustness of the proposed approach in large-scale problems. Experiments demonstrate the applicability of the proposed approach in practical applications. In further work, we will 
implement the proposed approach in practical robotic warehouses.

\end{document}